\documentclass{article}

\usepackage{times}
\usepackage{graphicx} 
\usepackage{subfigure} 

\usepackage{natbib}

\usepackage{algorithm}
\usepackage{algorithmic}

\usepackage{hyperref}

\usepackage{icml2014} 

\usepackage{epsfig}
\usepackage{amsmath}
\usepackage{amssymb}
\usepackage{amsfonts}
\usepackage{amsthm}  
\usepackage{wasysym}
\usepackage[ruled,vlined,algo2e]{algorithm2e}
\providecommand{\SetAlgoLined}{\SetLine}
\usepackage{color}
\usepackage{mathrsfs}

\newtheorem{thm}{Theorem}
\newtheorem{lem}{Lemma}

\theoremstyle{definition}

\DeclareMathOperator*{\argmin}{arg\,min}

\newcommand{\ie}{{\it i.e. }}
\newcommand{\eg}{{\it e.g. }}

\newcommand{\vs}{{\it vs. }}
\newcommand{\etal}{{\it et. al. }}

\newcommand{\rmnum}[1]{\romannumeral #1}

\icmltitlerunning{Regularization for Multiple Kernel Learning via Sum-Product Networks}

\begin{document} 

\twocolumn[
\icmltitle{Regularization for Multiple Kernel Learning via Sum-Product Networks}

\icmlauthor{Your Name}{email@yourdomain.edu}
\icmladdress{Your Fantastic Institute,
            314159 Pi St., Palo Alto, CA 94306 USA}
\icmlauthor{Your CoAuthor's Name}{email@coauthordomain.edu}
\icmladdress{Their Fantastic Institute,
            27182 Exp St., Toronto, ON M6H 2T1 CANADA}

\icmlkeywords{structured learning, multiple kernel learning, convex regularizer}

\vskip 0.3in
]

\begin{abstract}
In this paper, we are interested in constructing general graph-based regularizers for multiple kernel learning (MKL) given a structure which is used to describe the way of combining basis kernels. Such structures are represented by sum-product networks (SPNs) in our method. Accordingly we propose a new convex regularization method for MLK based on a path-dependent kernel weighting function which encodes the entire SPN structure in our method. Under certain conditions and from the view of probability, this function can be considered to follow multinomial distributions over the weights associated with product nodes in SPNs. We also analyze the convexity of our regularizer and the complexity of our induced classifiers, and further propose an efficient wrapper algorithm to optimize our formulation. In our experiments, we apply our method to ......
\end{abstract}

\section{Introduction}
In real world, information can be always organized under certain structures, which can be considered as the prior knowledge about the information. For instance, to understand a 2D scene, we can decompose the scene as ``scene $\rightarrow$ objects $\rightarrow$ parts $\rightarrow$ regions $\rightarrow$ pixels'', and reason the relations between them \cite{lubor}. Using such structures, we can answer questions like ``what and where the objects are'' \cite{Ladicky:2010:GCB:1888150.1888170} and ``what the geometric relations between the objects are'' \cite{DesaiRF_IJCV_2011}. Therefore, information structures are very important and useful for information integration and reasoning.

Multiple kernel learning (MKL) is a powerful tool for information integration, which aims to learn optimal kernels for the tasks by combining different basis kernels linearly \cite{simplemkl,DBLP:conf/icml/XuJYKL10,journals/jmlr/KloftBSZ11} or nonlinearly \cite{DBLP:conf/nips/Bach08,DBLP:conf/nips/CortesMR09,DBLP:conf/icml/VarmaB09} with certain constraints on kernel weights. In \cite{gonen11jmlr} a nice review on different MKL algorithms was given, and in \cite{Ryota2011} some regularization strategies on kernel weights were discussed. 

Recently, structure induced regularization methods have been attracting more and more attention \cite{DBLP:journals/corr/abs-1109-2397,Maurer:2012:SSG:2188385.2188408,Geer:arXiv1204.4813,DBLP:journals/csda/LinHP14}. For MKL, Bach \cite{DBLP:conf/nips/Bach08} proposed a hierarchical kernel selection (or more precisely, kernel decomposition) method for MKL based on directed acyclic graph (DAG) using structured sparsity-induced norm such as $\ell_1$ norm or block $\ell_1$ norm \cite{Jenatton:2011:SVS:1953048.2078194}. Szafranski \etal \cite{Szafranski:2010:CKL:1745449.1745456} proposed a composite kernel learning method based on tree structures, where the regularization term in their optimization formulation is a composite absolute penalty term \cite{Zhao_thecomposite}. Though the structure information of how to combine basis kernels are taken into account when constructing regularizers, however, the weights of nodes in the structures appear independently in these regularizers. This type of formulations actually weaken the connections between the nodes in the structures, making the learning rather easy.

To distinguish our work from previous research on regularization for MKL:
\begin{itemize}
\item[(1)] We utilize the sum-product networks (SPNs) \cite{DBLP:conf/uai/PoonD11} to describe the procedure of combining basis kernels. An SPN is a more general and powerful deep graphical representation consisting of only sum nodes and product nodes. Considering that the optimal kernel in MKL is created using summations and/or multiplications between non-negative weights and basis kernels, this procedure can be naturally described by SPNs. Notice that in general SPNs may not describe kernel embedding directly \cite{DBLP:journals/jmlr/ZhuangTH11a, journals/corr/StroblV13}. However, using Taylor series we still can approximate kernel embedding using SPNs.
\item[(2)] We accordingly propose a convex regularization method based on a new path-dependent kernel weighting function, which encodes the entire structures of SPNs. This function can be considered to follow multinomial distributions, involving much stronger connections between the node weights.
\end{itemize}
We also analyze the convexity of our regularizer and the Rademacher complexity of the induced MKL classifiers. Further we propose an efficient wrapper algorithm to solve our problem, where the weights are updated using gradient descent methods \cite{opac-b1131390}.

The rest of this paper is organized as follows. In Section \ref{sec:weighting_function}, we explain how to describe the kernel combination procedure using SPNs and our path-dependent kernel weighting function based on SPNs. In Section \ref{sec:method}, we provide the details of our regularization method, namely SPN-MKL, including the analysis of regularizer convexity, Rademacher complexity, and our optimization algorithm. We show our experimental results and comparisons among different methods on ......

\section{Path-dependent Kernel Weighting Function}\label{sec:weighting_function}
\subsection{Sum-Product Networks}\label{ssec:spn}
A sum-product network (SPN) is a rooted directed acyclic graph (DAG) whose internal nodes are sums and products \cite{DBLP:conf/uai/PoonD11}.

Given an SPN for MKL, we denote a \textbf{\em path} from the root node to a leaf node (\ie kernel) in the SPN as $\mathbf{m}\in\mathcal{M}$ where $\mathcal{M}$ consists of all the paths, and a \textbf{\em product node} as $v\in\mathcal{V}$ where $\mathcal{V}$ consists of all the nodes. Along each path $\mathbf{m}$, we call the sub-path between any pair of adjacent sum nodes or between a leaf node and its adjacent sum node a \textbf{\em layer}, and denote it as $\mathbf{m}_l$ $(l\geq1)$ and the number of layers along $\mathbf{m}$ as $N_{\mathbf{m}}$. We denote the number of product nodes in layer $\mathbf{m}_l$ as $N_{\mathbf{m}_l}$. We also denote the weights associated with path $\mathbf{m}$ and the weight of the $n^{th}$ product node in its layer $\mathbf{m}_l$ as $\boldsymbol{\beta}_{\mathbf{m}}$ and $\beta_{m_l^n}$, respectively, and $\boldsymbol{\beta}_{\mathbf{m}}$ is a vector consisting of all $\{\beta_{m_l^n}\}_{\forall m_l^n}$. There is no associated weight for any sum node in the SPN.

\begin{figure}[t]
\begin{center}
 \centerline{\includegraphics[width=\columnwidth]{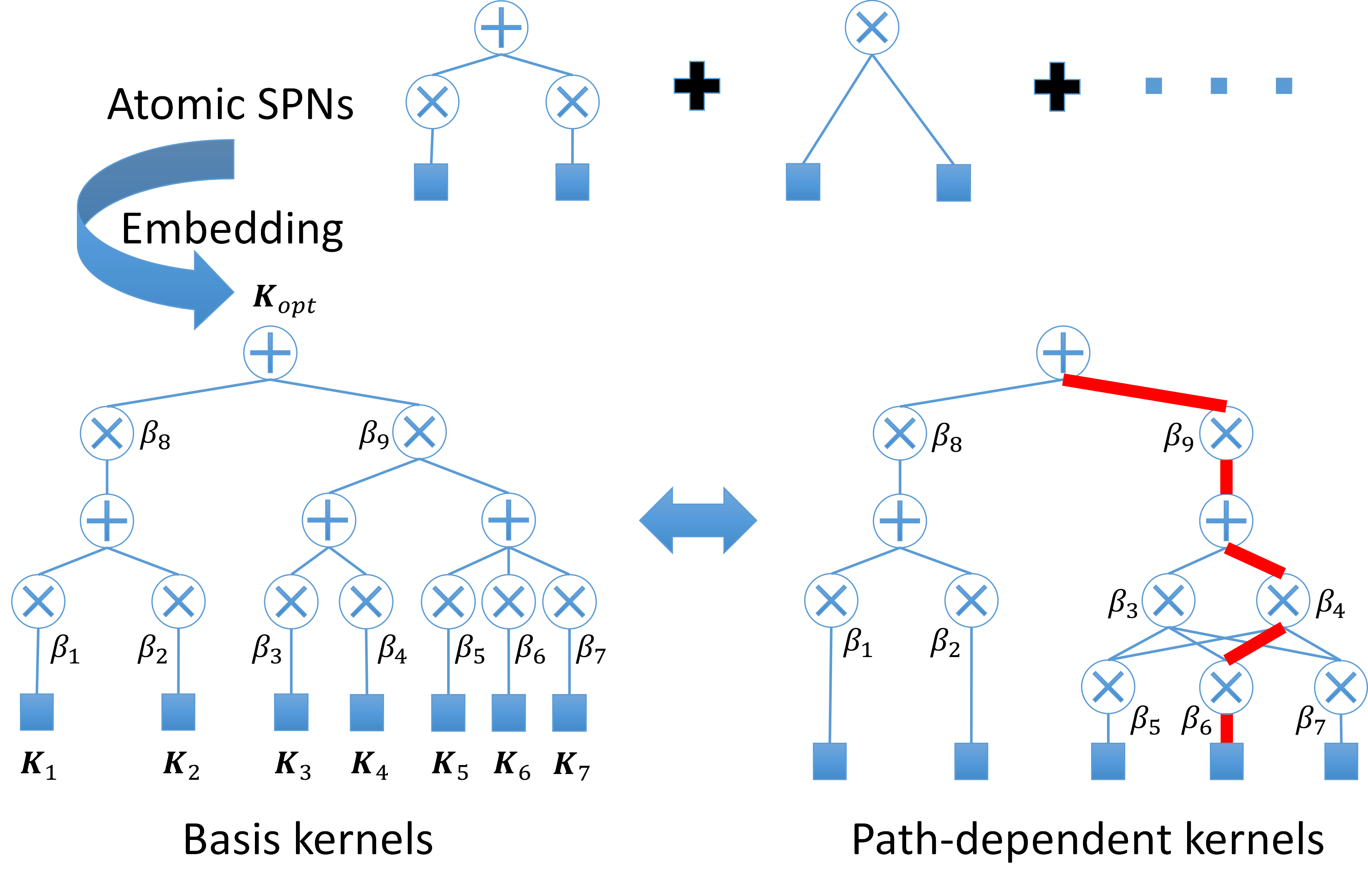}}
\end{center}\vspace{-3mm}
\caption{\footnotesize{An example (\ie bottom left) of constructing an SPN for basis kernel combination by embedding atomic SPNs into each other. All the weights (\ie $\beta$'s) associated with product nodes in the SPN are learned in our method. The red edges in the bottom right graph denote a path from the root to a path-dependent kernel. This figure is best viewed in color.}}\label{fig:spn}
\end{figure}

Fig. \ref{fig:spn} gives an example of constructing an SPN for basis kernel combination by embedding atomic SPNs into each other. \textbf{\em Atomic SPNs} in our method are the SPNs with {\em single} layer. Given an SPN as shown at the bottom left in Fig. \ref{fig:spn} and the node weights, we can easily calculate the optimal kernel as $\mathbf{K}_{opt}=\beta_8(\beta_1\mathbf{K}_1+\beta_2\mathbf{K}_2)+\beta_9(\beta_3\mathbf{K}_3+\beta_4\mathbf{K}_4)\circ(\beta_5\mathbf{K}_5+\beta_6\mathbf{K}_6+\beta_7\mathbf{K}_7)$, where $\circ$ denotes the entry-wise product between two matrices. Moreover, we can rewrite $\mathbf{K}_{opt}$ as $\mathbf{K}_{opt}=\beta_8(\beta_1\mathbf{K}_1+\beta_2\mathbf{K}_2)+\beta_9\prod_{i=3}^4\prod_{j=5}^7\beta_i\beta_j(\mathbf{K}_i\circ\mathbf{K}_j)$, whose combination procedure can be described using the SPN at the bottom right in Fig. \ref{fig:spn}. Here $\forall i, j, \mathbf{K}_i\circ\mathbf{K}_j$ is a \textbf{\em path-dependent kernel}. For instance, the corresponding kernel for the path denoted by the red edges at the bottom right figure is $\mathbf{K}_4\circ\mathbf{K}_6$. In fact, such kernel combination procedures for MKL can be always represented using SPNs in similar ways as shown at the bottom right figure.

Traditionally, SPNs are considered as probabilistic models and learned in an unsupervised manner \cite{NIPS2012_1484,2715,DBLP:conf/uai/PoonD11}. However, in our method we only utilize SPNs as representations to describe the kernel combination procedure, and learn the weights associated with their product nodes for MKL. In addition, from the aspect of structures for kernel combination, many existing MKL methods, \eg \cite{simplemkl,DBLP:conf/nips/CortesMR09,DBLP:conf/icml/XuJYKL10,Szafranski:2010:CKL:1745449.1745456}, can be considered as our special cases.

\subsection{Our Kernel Weighting Function}
Given an SPN and its associated weights $\beta$'s, we define our {\em path-dependent kernel weighting function} $g_{\mathbf{m}}(\boldsymbol{\beta}_{\mathbf{m}})$ as:
\begin{equation}
\forall\mathbf{m}\in\mathcal{M},\; g_{\mathbf{m}}(\boldsymbol{\beta}_{\mathbf{m}})=\prod_{l=1}^{N_{\mathbf{m}}}\prod_{n=1}^{N_{\mathbf{m}_l}}\left(\beta_{m_l^n}\right)^{\frac{1}{N_{\mathbf{m}}N_{\mathbf{m}_l}}}.
\end{equation} 
Taking the red path in Fig. \ref{fig:spn} for example, the kernel weighting function for this path is $g=\beta_9^{\frac{1}{2\times1}}\beta_4^{\frac{1}{2\times2}}\beta_6^{\frac{1}{2\times2}}$ with $N_{\mathbf{m}}=2$, $N_{\mathbf{m}_1}=1$, and $N_{\mathbf{m}_2}=2$.

Given an SPN and $\forall \mathbf{m}\in\mathcal{M}$, suppose $\forall m_l^n, 0\leq\beta_{m_l^n}\leq1$. Then from the view of probability, since $\forall \mathbf{m}\in\mathcal{M}$, $N_{\mathbf{m}}$ and $N_{\mathbf{m}_l}$ are constants, $g_{\mathbf{m}}$ actually follows a multinomial distribution with variables $\boldsymbol{\beta}_{\mathbf{m}}$ (ignoring the scaling factor). This is different from recent work \cite{DBLP:conf/icml/Gonen12}, where the kernel weights are assumed to follow multivariate normal distributions so that efficient inference can be performed. In contrast, our kernel weighting function is intuitively derived from the SPN structure, and under certain simple condition, it can guarantee the convexity of our proposed regularizer (see our Lemma \ref{lem:convexity}). 

\section{SPN-MKL}\label{sec:method}
\subsection{Formulation}
Given $N_{\mathbf{x}}$ training samples $\{(\mathbf{x}_i,y_i)\}$, where $\forall i, \mathbf{x}_i\in\mathbb{R}^d$ is an input data vector and $y_i\in\{1,-1\}$ is its binary label, we formulate our SPN-MKL for binary classification as follows:
\begin{eqnarray}
\label{eqn:primal-spn-mkl}
\lefteqn{\hspace{-22mm}\min_{\substack{\mathcal{B},\mathcal{W},b}} \hspace{1mm} \sum_{\mathbf{m}\in\mathcal{M}}\left\{\frac{\|\mathbf{w}_{\mathbf{m}}\|_2^2}{2\cdot g_{\mathbf{m}}(\boldsymbol{\beta}_{\mathbf{m}})} + \lambda\sum_{l=1}^{N_{\mathbf{m}}}\sum_{n=1}^{N_{\mathbf{m}_l}}\frac{\left(\beta_{m_l^n}\right)^{p_{m_l^n}}}{N_{\mathbf{m}}N_{\mathbf{m}_l}}\right\}}\nonumber\\
&& + C\sum_i\ell(\mathbf{x}_i, y_i; \mathcal{W}, b)\\
&&{
\begin{array}{ll}
\hspace{-29mm}\mbox{s.t.} & \forall \beta\in\mathcal{B}, \, \beta\geq0, \nonumber
\end{array}}
\end{eqnarray}
where $\mathcal{B}=\{\boldsymbol{\beta}_{\mathbf{m}}\}_{\forall\mathbf{m}\in\mathcal{M}}$ denotes the weight set, $\mathcal{W}=\{\mathbf{w}_{\mathbf{m}}\}_{\forall\mathbf{m}\in\mathcal{M}}$ denotes the classifier parameter set, $b$ denotes the bias term in the MKL classifier, $\lambda\geq0$, $C\geq0$, and $\mathcal{P}=\{p_{v}\}_{\forall v\in\mathcal{V}}$ are predefined constants. Function $\forall i, \ell(\mathbf{x}_i, y_i; \mathcal{W}, b)=\max\left\{0,1-y_i\left[\sum_{\mathbf{m}\in\mathcal{M}}\mathbf{w}_{\mathbf{m}}^T\phi_{\mathbf{m}}(\mathbf{x}_i)+b\right]\right\}$ denotes the hinge loss function, where $\forall\mathbf{m}\in\mathcal{M}, \phi_{\mathbf{m}}(\cdot)$ denotes a path-dependent kernel mapping function and $(\cdot)^T$ denotes the matrix transpose operator, and our decision function for a given data $\bar{\mathbf{x}}$ is $f(\bar{\mathbf{x}};\mathcal{B}, \mathcal{W},b)=\sum_{\mathbf{m}\in\mathcal{M}}\mathbf{w}_{\mathbf{m}}^T\phi_{\mathbf{m}}(\bar{\mathbf{x}})+b$. Moreover, we define $\forall\mathbf{m},\forall l,\forall n, \lim_{\beta_{m_l^n}\rightarrow0^+}\left\{\|\mathbf{w}_{\mathbf{m}}\|_2^2\left(\beta_{m_l^n}\right)^{-\frac{1}{N_{\mathbf{m}}N_{\mathbf{m}_l}}}\right\}=0$. This constraint guarantees the continuity of our objective function.

Note that unlike many existing MKL methods such as SimpleMKL \cite{simplemkl}, in Eq. \ref{eqn:primal-spn-mkl} there is no $\ell_p$ norm constraint on the node weights $\beta$'s. This makes the weight learning procedure more flexible, only dependent on the data and the predefined SPN structure.

\subsection{Analysis}
In this section, we analyze the properties of our proposed regularizer and the Rademacher complexity of the induced MKL classifier.

\begin{lem}\label{lem:convexity}
$\forall \mathbf{m}\in\mathcal{M}, f(\mathbf{w}_{\mathbf{m}}, \boldsymbol{\beta}_{\mathbf{m}})=\frac{\|\mathbf{w}_{\mathbf{m}}\|_2^2}{g_{\mathbf{m}}(\boldsymbol{\beta}_{\mathbf{m}})}$ is convex over both $\mathbf{w}_{\mathbf{m}}$ and $\boldsymbol{\beta}_{\mathbf{m}}$.
\end{lem}
\begin{proof}
Clearly, $f$ is continuous and differentiable with respect to $\mathbf{w}_{\mathbf{m}}$ and $\boldsymbol{\beta}_{\mathbf{m}}$, respectively. Given arbitrary $\mathbf{w}_{\mathbf{m}}^{(0)}$, $\mathbf{w}_{\mathbf{m}}^{(1)}$, $\boldsymbol{\beta}_{\mathbf{m}}^{(0)}\succeq\mathbf{0}$, and $\boldsymbol{\beta}_{\mathbf{m}}^{(1)}\succeq\mathbf{0}$, where $\succeq$ denotes the entry-wise $\geq$ operator, based on the definition of a convex function, we need to prove $f(\mathbf{w}_{\mathbf{m}}^{(1)},\boldsymbol{\beta}_{\mathbf{m}}^{(1)})\geq f(\mathbf{w}_{\mathbf{m}}^{(0)},\boldsymbol{\beta}_{\mathbf{m}}^{(0)})+(\mathbf{w}_{\mathbf{m}}^{(1)}-\mathbf{w}_{\mathbf{m}}^{(0)})^T\left.\frac{\partial f(\mathbf{w}_{\mathbf{m}},\boldsymbol{\beta}_{\mathbf{m}}^{(0)})}{\partial\mathbf{w}_{\mathbf{m}}}\right|_{\mathbf{w}_{\mathbf{m}}=\mathbf{w}_{\mathbf{m}}^{(0)}}+(\boldsymbol{\beta}_{\mathbf{m}}^{(1)}-\boldsymbol{\beta}_{\mathbf{m}}^{(0)})^T\left.\frac{\partial f(\mathbf{w}_{\mathbf{m}}^{(0)},\boldsymbol{\beta}_{\mathbf{m}})}{\partial\boldsymbol{\beta}_{\mathbf{m}}}\right|_{\boldsymbol{\beta}_{\mathbf{m}}=\boldsymbol{\beta}_{\mathbf{m}}^{(0)}}$. 

$\because \left.\frac{\partial f(\mathbf{w}_{\mathbf{m}},\boldsymbol{\beta}_{\mathbf{m}}^{(0)})}{\partial\mathbf{w}_{\mathbf{m}}}\right|_{\mathbf{w}_{\mathbf{m}}=\mathbf{w}_{\mathbf{m}}^{(0)}}=\frac{2\mathbf{w}_{\mathbf{m}}^{(0)}f(\mathbf{w}_{\mathbf{m}}^{(0)},\boldsymbol{\beta}_{\mathbf{m}}^{(0)})}{\|\mathbf{w}_{\mathbf{m}}^{(0)}\|_2^2}$, and $\forall m_l^n, \left.\frac{\partial f(\mathbf{w}_{\mathbf{m}}^{(0)},\boldsymbol{\beta}_{\mathbf{m}})}{\partial\beta_{m_l^n}}\right|_{\beta_{m_l^n}=\beta_{m_l^n}^{(0)}}=-\frac{f(\mathbf{w}_{\mathbf{m}}^{(0)},\boldsymbol{\beta}_{\mathbf{m}}^{(0)})}{N_{\mathbf{m}} N_{\mathbf{m}_l}\cdot\beta_{m_l^n}^{(0)}}$,

$\therefore$ By substituting above equations into our target, in the end we only need to prove that
\begin{eqnarray}\label{eqn:spn-convexity}
\lefteqn{\hspace{-3mm}f(\mathbf{w}_{\mathbf{m}}^{(1)},\boldsymbol{\beta}_{\mathbf{m}}^{(1)})- \frac{2\left(\mathbf{w}_\mathbf{m}^{(1)}\right)^T\mathbf{w}_{\mathbf{m}}^{(0)}f(\mathbf{w}_{\mathbf{m}}^{(0)},\boldsymbol{\beta}_{\mathbf{m}}^{(0)})}{\|\mathbf{w}_{\mathbf{m}}^{(0)}\|_2^2}}\nonumber\\
&&\hspace{8mm} +f(\mathbf{w}_{\mathbf{m}}^{(0)},\boldsymbol{\beta}_{\mathbf{m}}^{(0)})\sum_{l=1}^{N_{\mathbf{m}}}\sum_{n=1}^{N_{\mathbf{m}_l}}\frac{1}{N_{\mathbf{m}} N_{\mathbf{m}_l}}\left(\frac{\beta_{m_l^n}^{(1)}}{\beta_{m_l^n}^{(0)}}\right)\nonumber\\
&& \hspace{-10mm} \geq f(\mathbf{w}_{\mathbf{m}}^{(1)},\boldsymbol{\beta}_{\mathbf{m}}^{(1)})- \frac{2\left(\mathbf{w}_\mathbf{m}^{(1)}\right)^T\mathbf{w}_{\mathbf{m}}^{(0)}f(\mathbf{w}_{\mathbf{m}}^{(0)},\boldsymbol{\beta}_{\mathbf{m}}^{(0)})}{\|\mathbf{w}_{\mathbf{m}}^{(0)}\|_2^2}\nonumber\\
&& \hspace{30mm} + f(\mathbf{w}_{\mathbf{m}}^{(0)},\boldsymbol{\beta}_{\mathbf{m}}^{(0)}) \frac{g_{\mathbf{m}}(\boldsymbol{\beta}_{\mathbf{m}}^{(1)})}{g_{\mathbf{m}}(\boldsymbol{\beta}_{\mathbf{m}}^{(0)})}\nonumber\\
&& \hspace{-10mm} \geq\left\| \frac{\mathbf{w}_{\mathbf{m}}^{(1)}}{\sqrt{g_{\mathbf{m}}(\boldsymbol{\beta}_{\mathbf{m}}^{(1)})}}-\frac{\mathbf{w}_{\mathbf{m}}^{(0)}\sqrt{g_{\mathbf{m}}(\boldsymbol{\beta}_{\mathbf{m}}^{(1)})}}{g_{\mathbf{m}}(\boldsymbol{\beta}_{\mathbf{m}}^{(0)})}\right\|^2\geq0.
\end{eqnarray}
Since Eq. \ref{eqn:spn-convexity} always holds, our lemma is proven.
\end{proof}

\begin{lem}\label{lem:lower-bound}
Given an SPN for MKL, $\forall\mathbf{m}\in\mathcal{M}, \frac{\|\mathbf{w}_{\mathbf{m}}\|_2^2}{g_{\mathbf{m}}(\boldsymbol{\beta}_{\mathbf{m}})}\leq\sum_{l=1}^{N_{\mathbf{m}}}\sum_{n=1}^{N_{\mathbf{m}_l}}\frac{1}{N_{\mathbf{m}}N_{\mathbf{m}_l}}\frac{\|\mathbf{w}_{\mathbf{m}}\|_2^2}{\beta_{m_l^n}}$.
\end{lem}
\begin{proof}
\begin{eqnarray}
\frac{\|\mathbf{w}_{\mathbf{m}}\|_2^2}{g_{\mathbf{m}}(\boldsymbol{\beta}_{\mathbf{m}})}=\prod_{l,n}\left(\frac{\|\mathbf{w}_{\mathbf{m}}\|_2^2}{\beta_{m_l^n}}\right)^{\frac{1}{N_{\mathbf{m}}N_{\mathbf{m}_l}}}\leq\sum_{l,n}\frac{\|\mathbf{w}_{\mathbf{m}}\|_2^2}{N_{\mathbf{m}}N_{\mathbf{m}_l}\beta_{m_l^n}}.\nonumber
\end{eqnarray}
\end{proof}

From Lemma \ref{lem:convexity} and \ref{lem:lower-bound}, we can see that our proposed regularizer is actually the lower bound of a family of widely used MKL regularizers \cite{simplemkl,DBLP:conf/icml/XuJYKL10,gonen11jmlr,journals/jmlr/KloftBSZ11}, involving much stronger connections between node weights.

\begin{thm}[Convex Regularization]\label{thm:convexity}
Our regularizer in Eq. \ref{eqn:primal-spn-mkl} is convex if $\forall v\in\mathcal{V}, p_v\geq1$.
\end{thm}
\begin{proof}
When $\forall v\in\mathcal{V}, p_v\geq1$, $\forall m_l^n, \frac{\left(\beta_{m_l^n}\right)^{p_{m_l^n}}}{N_{\mathbf{m}}N_{\mathbf{m}_l}}$ is convex over $\mathcal{B}$. Then based on Lemma \ref{lem:convexity}, since the summation of convex functions is still convex, our regularizer is convex.
\end{proof}

\begin{thm}[Rademacher Complexity]\label{thm:complexity}
Denoting our MKL classifier learned from Eq. \ref{eqn:primal-spn-mkl} as $f(\mathbf{x}; \mathcal{B},\mathcal{W},b)=\sum_{\mathbf{m}}\mathbf{w}_{\mathbf{m}}^T\phi_{\mathbf{m}}(\mathbf{x})+b$ and our regularizer in Eq. \ref{eqn:primal-spn-mkl} as
\begin{eqnarray}\label{eqn:R}
\lefteqn{R(\mathcal{B},\mathcal{W};\lambda,\mathcal{P})=R_1+R_2=}\\
&&\sum_{\mathbf{m}\in\mathcal{M}}\frac{\|\mathbf{w}_{\mathbf{m}}\|_2^2}{2\cdot g_{\mathbf{m}}(\boldsymbol{\beta}_{\mathbf{m}})} + \lambda\sum_{\mathbf{m}\in\mathcal{M}}\sum_{l=1}^{N_{\mathbf{m}}}\sum_{n=1}^{N_{\mathbf{m}_l}}\frac{\left(\beta_{m_l^n}\right)^{p_{m_l^n}}}{N_{\mathbf{m}}N_{\mathbf{m}_l}},\nonumber
\end{eqnarray} 
the {\em empirical Rademacher complexity} of our classifier $\hat{R}(f)$ is upper-bounded by
\begin{equation}\label{eqn:complexity}
\frac{2A}{N_{\mathbf{x}}}\cdot\min_{\mathcal{B},\mathcal{W},b}\left\{R(\mathcal{B},\mathcal{W};1,\mathbf{1}) + C\sum_i\ell(\mathbf{x}_i, y_i; \mathcal{W}, b)\right\}\nonumber
\end{equation}
where $N_{\mathbf{x}}$ denotes the total number of training samples, constant $A=\left(\sum_{i=1}^{N_{\mathbf{x}}}\sum_{\mathbf{m}}\mathbf{K}_{\mathbf{m}}(\mathbf{x}_i,\mathbf{x}_i)\right)^{\frac{1}{2}}$, and $\forall\mathbf{m}, \forall i, \mathbf{K}_{\mathbf{m}}(\mathbf{x}_i,\mathbf{x}_i)=\phi_{\mathbf{m}}(\mathbf{x}_i)^T\phi_{\mathbf{m}}(\mathbf{x}_i)$ denotes the $i^{th}$ element along the diagonal of the path-dependent kernel matrix $\mathbf{K}_{\mathbf{m}}$.
\end{thm}
\begin{proof}
Given the Rademacher variables $\sigma$'s, based on the definition of Rademacher complexity, we have
\begin{eqnarray}
\lefteqn{\hat{R}(f)=\mathbb{E}_{\sigma}\left[\sup_{f\in\mathcal{F}(\mathcal{B},\mathcal{W})}\left|\frac{2}{N_{\mathbf{x}}}\sum_{i=1}^{N_{\mathbf{x}}}\sigma_if(\mathbf{x}_i;\mathcal{B},\mathcal{W},b)\right|\right]}\nonumber\\
&& \hspace{-7mm} =\mathbb{E}_{\sigma}\left[\sup_{f\in\mathcal{F}(\mathcal{B},\mathcal{W})}\left|\frac{2}{N_{\mathbf{x}}}\sum_{i=1}^{N_{\mathbf{x}}}\sigma_i\sum_{\mathbf{m}\in\mathcal{M}}\mathbf{w}_{\mathbf{m}}^T\phi_{\mathbf{m}}(\mathbf{x}_i)\right|\right]\nonumber\\
&& \hspace{-7mm} \leq\frac{4}{N_{\mathbf{x}}}\cdot\sup_{f\in\mathcal{F}}\left\{\left[\sum_{\mathbf{m}}\frac{\|\mathbf{w}_{\mathbf{m}}\|_2^2}{2\cdot g_{\mathbf{m}}(\boldsymbol{\beta}_{\mathbf{m}})}\right]^{\frac{1}{2}}\cdot\left[\sum_{\mathbf{m}}g_{\mathbf{m}}(\boldsymbol{\beta}_{\mathbf{m}})\right]^{\frac{1}{2}}\right\}\nonumber\\
&&\hspace{-7mm} \cdot \mathbb{E}_{\sigma}\left[\left\|\sum_{i=1}^{N_{\mathbf{x}}}\sum_{\mathbf{m}}\sigma_i\phi_{\mathbf{m}}(\mathbf{x}_i)\right\|\right]\nonumber\\
&& \hspace{-7mm} \leq\frac{2}{N_{\mathbf{x}}}\cdot\sup_{f\in\mathcal{F}}\left\{\sum_{\mathbf{m}}\frac{\|\mathbf{w}_{\mathbf{m}}\|_2^2}{2\cdot g_{\mathbf{m}}(\boldsymbol{\beta}_{\mathbf{m}})}+\sum_{\mathbf{m}}g_{\mathbf{m}}(\boldsymbol{\beta}_{\mathbf{m}})\right\}\cdot A\nonumber\\
&& \hspace{-7mm} \leq\frac{2A}{N_{\mathbf{x}}}\cdot\sup_{f\in\mathcal{F}}\left\{\sum_{\mathbf{m}}\left[\frac{\|\mathbf{w}_{\mathbf{m}}\|_2^2}{2\cdot g_{\mathbf{m}}(\boldsymbol{\beta}_{\mathbf{m}})}+\sum_{l=1}^{N_{\mathbf{m}}}\sum_{n=1}^{N_{\mathbf{m}_l}}\frac{\beta_{m_l^n}}{N_{\mathbf{m}}N_{\mathbf{m}_l}}\right]\right\}\nonumber\\
&& \hspace{-7mm} \leq\frac{2A}{N_{\mathbf{x}}}\cdot\min_{\mathcal{B},\mathcal{W},b}\left\{R(\mathcal{B},\mathcal{W};1,\mathbf{1}) + C\sum_i\ell(\mathbf{x}_i, y_i; \mathcal{W}, b)\right\}\nonumber
\end{eqnarray}
\end{proof}

From Theorem \ref{thm:complexity} we can see that with $\lambda=1$ and $\forall v\in\mathcal{V}, p_v=1$, minimizing our objective function in Eq. \ref{eqn:primal-spn-mkl} is equivalent to minimizing the upper bound of Rademacher complexity of our induced MKL classifier. To enhance the flexibility of our method, we allow $\lambda$ and $\mathcal{P}$ to be tuned according to datasets.

\subsection{Optimization}\label{ssec:optimization}
To optimize Eq. \ref{eqn:primal-spn-mkl}, we adopt a similar learning strategy as used in \cite{DBLP:conf/icml/XuJYKL10} by updating the node weights $\mathcal{B}$ and the classifier parameters $(\mathcal{W},b)$ alternatively. 

\subsubsection{Learning $(\mathcal{W},b)$ by fixing $\mathcal{B}$}
We utilize the dual form of Eq. \ref{eqn:primal-spn-mkl} to learn $(\mathcal{W},b)$. Letting $\boldsymbol{\alpha}\in\mathbb{R}^{N_{\mathbf{x}}}$ be the vector of Lagrange multipliers, and $\mathbf{y}\in\{-1,1\}^{N_{\mathbf{x}}}$ be the vector of binary labels, then optimizing the dual of Eq. \ref{eqn:primal-spn-mkl} is equivalent to maximizing the following problem:
\begin{eqnarray}
\label{eqn:dual-spn-mkl}
\lefteqn{\hspace{-13mm}\max_{\boldsymbol{\alpha}} \hspace{1mm} \mathbf{e}^T\boldsymbol{\alpha}-\frac{1}{2}(\boldsymbol{\alpha}\circ \mathbf{y})^T\left(\sum_{\mathbf{m}\in\mathcal{M}}g_{\mathbf{m}}(\boldsymbol{\beta}_{\mathbf{m}})\mathbf{K}_{\mathbf{m}}\right)(\boldsymbol{\alpha}\circ \mathbf{y})}\nonumber\\
&{
\begin{array}{ll}
\mbox{s.t.} & \mathbf{0}\preceq\boldsymbol{\alpha}\preceq C\mathbf{e}, \, \mathbf{y}^T\boldsymbol{\alpha}=0,
\end{array}}
\end{eqnarray}
where $\preceq$ denotes the entry-wise $\leq$ operator. Based on Eq. \ref{eqn:dual-spn-mkl}, the optimal kernel is constructed as $\mathbf{K}_{opt}=\sum_{\mathbf{m}\in\mathcal{M}}g_{\mathbf{m}}(\boldsymbol{\beta}_{\mathbf{m}})\mathbf{K}_{\mathbf{m}}$, and $\forall\mathbf{m}\in\mathcal{M}, \mathbf{w}_{\mathbf{m}}=g_{\mathbf{m}}(\boldsymbol{\beta}_{\mathbf{m}})\sum_i\alpha_iy_i\phi_{\mathbf{m}}(\mathbf{x}_i)$.

Therefore, the updating rule for $\|\mathbf{w}_{\mathbf{m}}\|_2^2$ is:
\begin{equation}\label{eqn:spn-w}
\forall \mathbf{m}\in\mathcal{M}, \, \|\mathbf{w}_{\mathbf{m}}\|_2^2=g_{\mathbf{m}}(\boldsymbol{\beta}_\mathbf{m})^2\left(\boldsymbol{\alpha}\circ\mathbf{y}\right)^T\mathbf{K}_{\mathbf{m}}\left(\boldsymbol{\alpha}\circ\mathbf{y}\right).
\end{equation}

\subsubsection{Learning $\mathcal{B}$ by fixing $(\mathcal{W},b)$}
At this stage, minimizing our objective function in Eq. \ref{eqn:primal-spn-mkl} is equivalent to minimizing $R(\mathcal{B},\mathcal{W};\lambda,\mathcal{P})$ in Eq. \ref{eqn:R}, provided that $\forall \beta\in\mathcal{B}, \beta\geq1$. For further usage, we rewrite $R_2$ in Eq. \ref{eqn:R} as follows:
\begin{equation}
R_2=\sum_{v\in\mathcal{V}}\left(\sum_{\mathbf{m}\in\mathcal{M}(v)}\sum_{l=1}^{N_{\mathbf{m}}}\sum_{n=1}^{N_{\mathbf{m}_l}}\frac{\lambda}{N_{\mathbf{m}}N_{\mathbf{m}_l}}\right)\beta_{v}^{p_v},
\end{equation}
where $\mathcal{M}(v)$ denotes all the paths which pass through product node $v$. 

Due to the complex structures of SPNs, in general there may not exist close forms to update $\mathcal{B}$. Therefore, we utilize gradient descent methods to update $\mathcal{B}$.

\textbf{\em (\rmnum{1}) Convex Regularization with $\forall v\in\mathcal{V}, p_v\geq1$}

Since in this case our objective function is already convex, we can calculate its gradient directly and use the following rule to update $\mathcal{B}$: $\forall v\in\mathcal{V}$,
\begin{equation}\label{eqn:update-beta} \beta_v^{(k+1)}=\left[\beta_v^{(k)}-\eta_{k+1}\left(\nabla_{\beta_v} R_1(\mathcal{B}^{(k)})+\nabla_{\beta_v} R_2(\mathcal{B}^{(k)})\right)\right]_{+}
\end{equation}
where $\nabla_{\beta_v}$ denotes the first-order derivative operation over variable $\beta_v$, $(\cdot)^{(k)}$ denotes the value at the $k^{th}$ iteration, $\eta_{k+1}\geq0$ denotes the step size at the $(k+1)^{th}$ iteration, and $[\cdot]_+=\max\{0,\cdot\}$.

\textbf{\em (\rmnum{2}) Non-Convex Regularization with $\exists v\in\mathcal{V}, 0<p_v<1$}

In this case, since our objective function can be decomposed into summation of convex (\ie in $R_2$ all terms with $p_v\geq1$) and concave (\ie in $R_2$ all terms with $0<p_v<1$) functions, we can utilize Concave-Convex procedure (CCCP) \cite{Yuille:2003:CP:773700.773708} to optimize it. Therefore, the weight updating rule for nodes with $0<p_v<1$ is changed to:
\begin{equation}\label{eqn:cccp}
\beta_v^{(k+1)}=\argmin_{\beta_v\geq0}\left\{R_1+\beta_v\nabla_{\beta_v} R_2(\mathcal{B}^{(k)})\right\}.
\end{equation}
Again Eq. \ref{eqn:update-beta} can be reused to solve Eq. \ref{eqn:cccp} iteratively.

\begin{algorithm}[t]
\SetAlgoLined
\SetKwInOut{Input}{Input}\SetKwInOut{Output}{Output}
\Input{$\{(\mathbf{x}_i,y_i)\}_{i=1,\cdots,N_{\mathbf{x}}}$, an SPN, $\{p_{v}>0\}_{\forall v\in\mathcal{V}}$, $\{\mathbf{K}_{\mathbf{m}}\}_{\forall\mathbf{m}\in\mathcal{M}}$, $C$}
\Output{$\boldsymbol{\alpha}$, $\mathcal{B}=\{\beta_v\}_{\forall v\in\mathcal{V}}$}
\BlankLine
Initialize the kernel weights so that $\forall v\in\mathcal{V}, \beta_v\geq0$;\\
\Repeat{Converge}{
Update $\boldsymbol{\alpha}$ using Eq. \ref{eqn:dual-spn-mkl} while fixing $\mathcal{B}$;\\
(For multiclass cases, update $\{\boldsymbol{\alpha}_c\}_{c\in\mathcal{C}}$ using Eq. \ref{eqn:dual-spn-mkl-multiclass} while fixing $\mathcal{B}$;)\\
Update $\forall\mathbf{m}, \|\mathbf{w}_{\mathbf{m}}\|_2^2$ using Eq. \ref{eqn:spn-w} while fixing $\boldsymbol{\alpha}$ and $\mathcal{B}$;\\
(For multiclass cases, update $\forall\mathbf{m}, \|\mathbf{w}_{\mathbf{m}}\|_2^2$ using Eq. \ref{eqn:spn-w-multiclass} while fixing $\{\boldsymbol{\alpha}_c\}_{c\in\mathcal{C}}$ and $\mathcal{B}$;)\\
\ForEach{$v\in\mathcal{V}$}{
\eIf{$p_v\geq1$}{
Update $\beta_v$ using Eq. \ref{eqn:update-beta} while fixing $\mathbf{w}$;
}
{
\Repeat{Converge}{
Update $\beta_v^{(k+1)}$ using Eq. \ref{eqn:update-beta} and Eq. \ref{eqn:cccp} while fixing $\mathbf{w}$;
}
}
}
}
\Return $\boldsymbol{\alpha}$, $\mathcal{B}$\;
\caption{SPN-MKL learning algorithm}\label{alg:spn-mkl}
\end{algorithm}

To summarize, {\em as long as $\forall v\in\mathcal{V}, p_v>0$, we can always optimize our objective function.} We show our learning algorithm for binary SPN-MKL in Alg. \ref{alg:spn-mkl}. Note that once the weight of any product node is equal to 0, it will always keep zero, which indicates that the product node and all the paths that go through it can be deleted from the SPN permanently. This property can be used to simplify the SPN structure and accelerate the learning speed of our SPN-MKL. 

\subsection{Multiclass SPN-MKL}
For multiclass tasks, we generate a single optimal kernel for all the classes, and correspondingly modify Eq. \ref{eqn:dual-spn-mkl} and Eq. \ref{eqn:spn-w} for binary SPN-MKL without changing other steps. Using the ``one \vs the-rest'' strategy, the modification is shown as follows:
\begin{eqnarray}
\label{eqn:dual-spn-mkl-multiclass}
\lefteqn{\hspace{-0mm}\max_{\{\boldsymbol{\alpha}_c\}_{c\in\mathcal{C}}} \hspace{1mm} \sum_{c\in\mathcal{C}}\mathbf{e}^T\boldsymbol{\alpha}_c}\\
&& \hspace{-56mm} -\frac{1}{2}(\boldsymbol{\alpha}_c\circ \mathbf{y}_c)^T\left(\sum_{\mathbf{m}\in\mathcal{M}}g_{\mathbf{m}}(\boldsymbol{\beta}_{\mathbf{m}})\mathbf{K}_{\mathbf{m}}\right)(\boldsymbol{\alpha}_c\circ \mathbf{y}_c)\nonumber\\
&{
\begin{array}{ll}
\mbox{s.t.} & \forall c\in\mathcal{C}, \, \mathbf{0}\preceq\boldsymbol{\alpha}_c\preceq C\mathbf{e}, \, \mathbf{y}_c^T\boldsymbol{\alpha}_c=0,\nonumber
\end{array}}
\end{eqnarray}
\begin{equation}\label{eqn:spn-w-multiclass}
\forall \mathbf{m}, \|\mathbf{w}_{\mathbf{m}}\|_2^2=g_{\mathbf{m}}(\boldsymbol{\beta}_\mathbf{m})^2\sum_{c\in\mathcal{C}}\left(\boldsymbol{\alpha}_c\circ\mathbf{y}_c\right)^T\mathbf{K}_{\mathbf{m}}\left(\boldsymbol{\alpha}_c\circ\mathbf{y}_c\right),
\end{equation}
where $c\in\mathcal{C}$ denotes a class label $c$ in a label set $\mathcal{C}$, $\boldsymbol{\alpha}_c$ denotes a clss-specific Lagrange multipliers, and $\mathbf{y}_c$ denotes a binary label vector: if $\forall i, y_i=c$, then the $i^{th}$ entry in $\mathbf{y}_c$ is set to 1, otherwise, 0.

The learning algorithm for multiclass SPN-MKL is listed in Alg. \ref{alg:spn-mkl} as well.

\section{Experiments}
\label{sec:exp}

\bibliography{example_paper}
\bibliographystyle{icml2014}

\end{document}